\documentclass[12pt]{colt2020}

\usepackage{times}

\usepackage{graphicx}
\usepackage{amsmath}
\usepackage{amsfonts}
\usepackage{amssymb}
\usepackage{caption2}
\usepackage{algpseudocode,algorithm,algorithmicx}
\usepackage{booktabs}
\usepackage{multirow, makecell}

\newcommand{\be}{\begin{equation}}
\newcommand{\ee}{\end{equation}}

\newcommand{\bea}{\begin{eqnarray}}
\newcommand{\eea}{\end{eqnarray}}

\newcommand{\bi}{\begin{itemize}}
\newcommand{\ei}{\end{itemize}}

\newcommand{\ben}{\begin{enumerate}}
\newcommand{\een}{\end{enumerate}}

\newcommand{\bef}{\begin{figure}[tbp]}
\newcommand{\enf}{\end{figure}}

\newcommand{\bt}{\begin{tabular}{lcllcl}}
\newcommand{\et}{\end{tabular}}

\newcommand{\bd}{\begin{description}}
\newcommand{\ed}{\end{description}}

\newcommand{\eref}[1]{(\ref{#1})}       % Equation reference.

\newcommand{\dfn}{\stackrel{\triangle}{=}}  % Equal by definition.
\newcommand{\comb}[2]{\left ( \begin{array}{c}
 {#1} \\
 {#2} \end{array} \right )}

\newcommand{\Rspace}{\mathbb{R}}

\newcommand{\Regret}{{\cal{R}}}

\renewcommand{\t}{\theta}
\newcommand{\T}{\Theta}
\newcommand{\TT}{\Lambda}
\renewcommand{\S}{\mathcal{S}}
\newcommand{\grid}{\Psi}
\newcommand{\e}{\varepsilon}
\newcommand{\A}{\mathcal{A}}
\newcommand{\X}{\mathcal{X}}
\newcommand{\Y}{\mathcal{Y}}
\newcommand{\R}{{\mbox{Regret}}}

\title[Logistic Regression Regret]{Logistic Regression Regret:  What's the Catch?}

\coltauthor{
 \Name{Gil I. Shamir} \Email{gshamir@google.com}\\
 \addr Google
}

\begin{document}

\maketitle

\begin{abstract}
We address the problem of the achievable regret rates with online logistic regression.  We derive lower bounds with logarithmic
regret under  $L_1$, $L_2$, and $L_\infty$ constraints on the parameter values.  The bounds are dominated by
$d/2 \log T$, where $T$ is the horizon and $d$ is the dimensionality of the parameter space.  We
show their achievability for $d=o(T^{1/3})$ in all these cases with Bayesian methods, that achieve them up to a $d/2 \log d$ term.
%The bounds consist also of additional logarithmic terms that depend on the dimensionality of the problem $d$ and the constraints on the parameter space. 
Interesting different behaviors are shown for larger dimensionality.  Specifically, on the negative side,
 if $d = \Omega(\sqrt{T})$, any algorithm is
guaranteed regret of $\Omega(d \log T)$ (greater than $\Omega(\sqrt{T})$)
under $L_\infty$ constraints on the parameters (and the example features).
On the positive side, under $L_1$ constraints on the parameters,
there exist Bayesian algorithms that can achieve regret that is sub-linear in $d$ for 
the asymptotically larger values of $d$.  For $L_2$ constraints, it is shown that for large enough $d$, 
the regret remains linear in $d$ but no longer logarithmic in $T$.
Adapting the \emph{redundancy-capacity\/} theorem from information theory,
we demonstrate a principled methodology based on
grids of parameters to derive lower bounds.
Grids are also utilized to derive some upper bounds.  
Our results strengthen results by \cite{kakade05} and \cite{foster18} for upper bounds for this problem, introduce novel
lower bounds, and adapt a methodology that can be used to obtain such bounds for other
related problems.  They also give a novel characterization of the asymptotic behavior when the dimension of the parameter space is
allowed to grow with $T$.  They additionally establish connections to the information theory literature, demonstrating that the actual regret
for logistic regression
depends on the richness of the parameter class, where even within this problem, richer classes lead to greater regret.
\end{abstract}


\section{Introduction}
\label{sec:introduction}

Logistic regression plays a significant role in many learning applications, where a set of parameters representing the effects of different
features on the outcome (label) is learned from a training data set with known \emph{labels}.
The learned parameters are then used to predict the true labels of, yet unseen, data examples.
Examples include predicting the probability some person carries some disease based on features that are, e.g., hereditary or
environmental; or predicting the click-through-rate of ads shown in online advertising.
Many applications may require to operate in the online learning (or online convex optimization) setting.  In this setting, an algorithm
consumes the data in rounds.  At round $t$, predictions can be based on all examples seen up to round $t-1$, including on their true labels
(but not on data beyond round $t-1$), to
predict the label of the example at round $t$.  

The performance of an online algorithm is measured by its \emph{regret}, which is defined as the extra loss it incurs beyond that of an algorithm
that is playing, at all rounds, some \emph{comparator\/} value $\t^* \in \T$, where $\T$ is a predefined space of possible values.
The values of the parameters $\t^*$ can be those that
minimize the cumulative loss over all rounds up to the horizon $T$.  While regret is defined
for the online setting, it is directly connected to the \emph{convergence rate}, which measures an expected loss on an unseen example
at round $T+1$, based on training on the first $T$ examples. 

{\bf Paper Outline:} In Section~\ref{sec:contrib}, we outline our contributions.  We present a summary of related work in
Section~\ref{sec:related}.  Section~\ref{sec:notedef} formulates the problem.  In Section~\ref{sec:background}, we frame and extend results from the literature, setting them to prove our results.
Section~\ref{sec:lb} describes regret lower bounds for any algorithm. Section~\ref{sec:ub} shows upper bounds that can be
achieved with \emph{Bayesian mixture\/} algorithms and apply to logistic regression when the \emph{feature vector\/}
$x_t$ is observed prior to predicting a label.

\section{Summary of Contributions and Methods}
\label{sec:contrib}

We consider several settings with a $d$-dimensional parameter space with some limit $B$ on some norm of the parameters.
Specifically,  $\T \dfn \{\t^* : \|\t^*\|_\rho \leq B \}$, 
%i.e., $\t^* \in \T$, and $\|\t^*\|_p \leq B$;
for $\rho = 1, 2, \infty$.  Define $\gamma = B / (\log T)$ as the count of $\log T$ units constituting $B$.
We focus on the case in which
the norm of the \emph{example\/} (or, feature value vector) $x_t$ at $t$ is bounded in $L_\infty$, i.e., $|x_{t,i}| \leq 1, \forall i = 1,2, \ldots, d$,
(or $\|x_t\|_\infty \leq 1$).  (This setup generalizes the practical setup with binary features).
However,
with proper adjustments (which decrease the bounds), the results transform also to the more restrictive $\|x_t\|_2 \leq 1$.

Our contributions include:
\bi
 \vspace{-.3cm}
 \item Comprehensive characterization of the regret for the logistic regression problem, including the asymptotic behavior in the
 dimensionality $d$, showing regret bounds logarithmic in $T$ and linear in $d$ for lower regions of $d$.
 \vspace{-.3cm} 
 \item Novel bounds that lead to this characterization, especially, lower bounds showing limitations on regret in the different
 settings.
 \vspace{-.3cm} 
 \item Specific negative results that demonstrate that in cases such as $L_\infty$ constraints, for $d = \Omega(\sqrt{T})$, we
 are guaranteed regret rates of at least $\Omega (\sqrt{T} \log T)$.
 \vspace{-.3cm} 
 \item Specific positive results that demonstrate that for upper regions of $d$, there exist Bayesian algorithms with regret rates $o(d)$ (for $L_1$ constraints
 with $d = \Omega(B\sqrt{T})$), as well as regret rates that are linear in $d$, and no longer logarithmic in $T$ (for $L_2$ constraints with
 $d = \Omega (B^2 T)$).
 \vspace{-.3cm} 
 \item Adaptation of a principled methodology from the information theory literature, that allows derivation of lower bounds for this and
 related problems.
 \vspace{-.3cm}  
\ei
The sub-linear regret in $d$ for $L_1$ is very interesting for logistic regression because the dot product,
%often of binary features,
%($x_{t,i} \in \{0, 1\}$),
used for prediction, is a linear combination of the parameters, making $L_1$ constraints very realistic,
especially in sparse real-worlds problems that
have binary feature vectors (see, e.g., \cite{mcmahan13}). 

Our results characterize the behavior of the regret for the various regions of $d$.
For smaller $d$,
we show lower bounds of $(d/2) \log (T/d)$, $(d/2) \log (T/d^2)$, and $(d/2) \log (T/d^3)$ for the
cases where $\|\t^* \|_\rho \leq B$ and $\rho$ is $\infty, 2, 1$, respectively, and upper bounds of
$(d/2) \log (B^2 T)$, $(d/2) \log (B^2 T/d)$, and $(d/2) \log (B^2 T/d^2)$ for the respective norm constraints.
(An additional $d$ term in the denominator of the logarithm in all bounds applies to the setting in which $\|x_t\|_2 \leq 1$.)
The difference between the constraints on different norms illustrates that regret is a function
of the richness of $\T$.  The richer is $\T$ (e.g., $L_\infty$ constraints are richer than $L_2$, which are richer than $L_1$)
the greater is the regret.
As the dimension $d$ is allowed to grow, the bounds change when the denominator of the logarithmic term above
equals the numerator.  They lead to different regions of the different bounds, with different behavior in each region.
Table~\ref{tab:results} shows the different lower and upper bounds for different $d$ and different norm
constraints on the space $\T$, and summarizes the results in Theorems~\ref{th:binary_lb}-\ref{the:bayes_ub1}.
For simplicity of the table, we omitted the lower limit on $d$ for each row, but it should be understood as $\T(n)$ where the upper
limit of the previous row is $o(n)$ (for the first row in each block, the previous $n = 1$).  The lower bounds that are $\T(T^\alpha)$ should be understood
as $\T ( T^{\alpha(1-\e)})$ for 
some small $\e > 0$  which can be as small as $O(\log \log T / (\log T))$.  This is again omitted for simplicity.
For the setting in which $\|x_t\|_2 \leq 1$, the additional $d$ term in the denominators of the logarithm leads to earlier
transitions between regions of $d$, for all cases (as well as adding a $d/2$ upper regret region for $L_{\infty}$).
Table~\ref{tab:comp} compares results in this paper to previously reported results (described in more detail in Section~\ref{sec:related}).
We omit middle ranges of $d$ that are
in Table~\ref{tab:results}.  Some results in this paper are extended from \cite{kakade05}
and adapted to the setup $\|x_t\|_\infty \leq 1$.
A footnote marks these with a proper explanation.
For multi labels, we use $\t^{*(m)}$ to denote the $d$-dimensional projection of the parameter space for label $m$.
Results for $L_1$ and $L_\infty$ (that were not directly derived)
are generalized from results that were derived for $L_2$ and are described in the ``previous results'' column.
\vspace{-.5cm}
\begin{table*}[h]
        \caption{Summary of regret bounds}
  \label{tab:results}
  \centering
        \begin{small}
        %\begin{sc}

  \begin{tabular} {lp{2.5cm}p{3.5cm}p{3.5cm}}
    \toprule
     Norm Constraint &  \footnotemark[1]  Dimension $d$& \footnotemark[2] Lower Bound  & Upper Bound \\
    \midrule
      $L_1: \|\t^*\|_1 \leq B$ & 
       $o\left ((\gamma T)^{1/3} \right ) $ &  $\frac{d}{2} \log \frac{\gamma T}{d^3}$ & $\frac{d}{2} \log \frac{B^2 T}{d^2}$ \\
      & $o\left (B \sqrt{T} \right )$ & $\T \left ( (\gamma T)^{1/3}\right )$ & $\frac{d}{2} \log \frac{B^2 T}{d^2}$ \\
       & $\Omega \left (B \sqrt{T} \right )$  & $\T \left ( (\gamma T)^{1/3}\right )$ & $o(d)$ \\

    \midrule
      $L_2: \|\t^*\|_2 \leq B$ & $o\left (\sqrt{\gamma T}\right ) $ &  $\frac{d}{2} \log \frac{\gamma T}{d^2}$ & $\frac{d}{2} \log \frac{B^2 T}{d}$ \\
       & $o\left (B^2 T \right )$ & $\T \left ( \sqrt{\gamma T}\right )$ & $\frac{d}{2} \log \frac{B^2 T}{d}$ \\
       & $\Omega \left (B^2 T \right )$ & $\T \left ( \sqrt{\gamma T}\right )$ & $\frac{d}{2}$ \\

    \midrule
      $L_\infty: \|\t^*\|_\infty \leq B$ & 
        $o\left (\gamma T\right ) $ &  $\frac{d}{2} \log \frac{\gamma T}{d}$ & $\frac{d}{2} \log \left (B^2 T \right )$ \\
        & $\Omega \left (\gamma T\right ) $ &  $\T \left ( \gamma T \right )$ & $\frac{d}{2} \log \left (B^2 T \right )$ \\

    \bottomrule

  \end{tabular}
  %\end{sc}
  \end{small}
\end{table*}

\footnotetext[1]{The dimension column shows an upper limit on the shown range.  The lower limit should be understood as $\T(n)$ where  the upper
limit for the previous row is $o(n)$, ($n=1$ for the row ``previous'' to the first one in a block). }
\footnotetext[2]{Lower bounds that are $\T(T^\alpha)$ should be understood
as $\T ( T^{\alpha(1-\e)})$ for some small $\e > 0$.}

\begin{table*}[h]
        \caption{Comparison of regret bounds in this paper with previously reported bounds}
  \label{tab:comp}
  \centering
        \begin{small}
        %\begin{sc}

  \begin{tabular} {p{2.9cm}p{0.1cm}p{6.6cm}p{4.7cm}}
    \toprule
     Problem Setting & & Previous Results & This Paper \\
    \midrule
      Binary Labels, \newline $d = 1$  & : &
      $\Regret = \frac{1}{2} \log T$ [\citet{davisson73}] \newline [\citet{krichevsky81}] \newline [\citet{mcmahan12}] \footnotemark[1]& \\
      \midrule
      Multi Labels \footnotemark[2] \newline $\Omega(1) < m = o(T)$ \newline $d = 1$  & : &
      $\Regret = \frac{m}{2} \log (T/m)$ [\citeauthor{krichevsky81}] \newline
      [\citet{orlitsky04}] \newline [\citet{shamir06}] &
       \\
      \midrule
      
      Binary Labels \footnotemark[3] \newline Multi Dimensions  $d$ \newline $L_1: \|\t^*\|_1 \leq B$ & : &
      $O( B \sqrt{d T})$ [\citet{xiao10}] \newline
      $\Regret \leq \frac{d}{2} \log \left ( 1 + T \right )$ [\citet{kakade05}] \newline
      $\Regret \leq \frac{d}{2} \log \left (\frac{B^2 T}{d} + e \right )$ [\citeauthor{kakade05}]  \footnotemark[4] \newline
      $\Regret \leq 5 d \log \left ( \frac{B T}{d} + e \right )$ [\citet{foster18}]  &     
       $d = o\left ((\gamma T)^{1/3} \right ) $: \newline 
       $~~~~~\frac{d}{2} \log \frac{\gamma T}{d^3}  \leq \Regret \leq \frac{d}{2} \log \frac{B^2 T}{d^2}$ 
       \newline
       $d = \Omega \left (B \sqrt{T} \right )$ : \newline
       $~~~~~\T \left ( (\gamma T)^{1/3}\right ) \leq \Regret = o(d)$       
       \\
       \midrule
       
      Binary Labels \newline Multi Dimensions $d$ \newline $L_2: \|\t^*\|_2 \leq B$ & : &
      $O( B \sqrt{d T})$ [\citet{xiao10}] \newline
      $\Regret \leq \frac{d}{2} \log \left ( 1 + T \right )$ [\citet{kakade05}] \newline
      $\Regret \leq \frac{d}{2} \log \left ( \frac{B^2 T}{d} + e \right )$ [\citeauthor{kakade05}]  \footnotemark[4] \newline
      $\Omega(d) \leq \Regret \leq 5 d \log \left ( \frac{B T}{d} + e \right )$ [\citeauthor{foster18}] 
       & $d = o\left (\sqrt{\gamma T}\right ) $: \newline 
       $~~~~~\frac{d}{2} \log \frac{\gamma T}{d^2}  \leq \Regret \leq \frac{d}{2} \log \frac{B^2 T}{d}$ 
       \newline
       $d = \Omega \left (B^2 T \right )$ : \newline
       $~~~~~\T \left ( \sqrt{\gamma T}\right ) \leq \Regret \leq \frac{d}{2}$ \\
       \midrule
        
       Binary Labels \footnotemark[3]  \newline Multi Dimensions $d$ \newline $L_\infty: \|\t^*\|_\infty \leq B$ & : & 
       $O( d B \sqrt{T})$ [\citet{mcmahan17}] \newline
       $\Regret \leq \frac{d}{2} \log \left (B^2 T \right )$ [\citet{kakade05}]  \footnotemark[4] \newline
       $\Omega(d)$ [\citet{foster18}]  & 
       $d = o\left (\gamma T\right )  $: \newline 
       $~~~~~\frac{d}{2} \log \frac{\gamma T}{d}  \leq \Regret \leq \frac{d}{2} \log \left (B^2 T \right )$ 
       \newline
       $d = \Omega \left (\gamma T\right )$ : \newline
       $~~~~~\T \left ( \gamma T \right ) \leq \Regret \leq \frac{d}{2} \log \left (B^2 T \right )$
       \\
       \midrule
       Multi Labels $m$ \newline Multi Dimensions $d$ & : & 
       $L_2$ constraints: \newline
       $\Regret \leq 5 m d \log \left ( \frac{B T}{d m} + e \right )$ [\citet{foster18}]  &
       $\|\t^{*(m)}\|_\infty \leq B$: \newline
       $\Regret \geq \frac{d(m-1)}{2} \log \left (\frac{T}{d\cdot m} \right )$ \\
       \midrule
       Binary Labels, \newline Multi-$d$, proper & : & $\Omega(\sqrt{B T})$ [\citet{hazan14}] & \\
   \bottomrule

  \end{tabular}
  %\end{sc}
  \end{small}
\end{table*}

To prove lower bounds, we adapt techniques based on the \emph{redundancy-capacity theorem\/}
(see, e.g., \cite{davisson73, merhav95, shamir06}) from the information theory literature.  Specifically, we
set a \emph{grid\/} of points in the parameter space that are distinguishable by the observed label sequence for some
example sequence.  The logarithm of the cardinality of the grid is a lower bound on the regret.
The concept of distinguishability was used somewhat differently by \cite{hazan14} to prove regret lower bounds.
Upper bounds for $L_2$ and $L_\infty$, but not for $L_1$, can be derived by manipulating the \emph{Bayesian mixture\/} approach in
\cite{kakade05}, (adjusted to our setup).  Using a normal prior with \emph{large variance\/} can attain the proper rates, with
the respective constants. However, for $\|\t^*\|_1 \leq B$, we combine this approach with the method
of grids, applying a discrete uniform prior on some $\T_m \subseteq \T$.
Applying the method in \cite{kakade05}, $\log |\T_m|$ initially dominates an upper bound, with additional contribution
from the effective \emph{quantization\/} of the parameters by the mixture only on a discrete subset of the space.  This method can also be used for $L_2$
and $L_\infty$, and achieves a similar bound for $L_\infty$, but a weaker one for $L_2$.

%For the latter, it achieves the same bound as the other method.  For the former, however, the normal prior 
%is a better prior for the constraints on the parameter.  The same order is still achieved with both methods, but better
%second order terms are attained with the normal prior.
%For deriving upper bounds, the grids are dense enough to achieve covering of the parameter
%space.  Using Bayesian mixture over the grid leads to the upper bounds, where the negative logarithm of the prior at the
%hindsight best parameter value is sufficient to give a tight upper bound.  The additional regret from \emph{quantizing} the best
%parameter to the nearest grid point, is shown to be of negligible order relative to the bound.  
%The technique thus does use Bayesian model averaging, but differently from the methods described above.
%Using this technique, we demonstrate that for the binary label case, the regret is indeed
%$d/2 \log (B T/d)$ to first order, (at least when the parameters are restricted by an $L_1$ ball),
%where matching lower and upper bounds are shown.  Upper bounds are shown in three cases, where
%the value of $\t_t$ is restricted to a ball with radius $D$ in $L_1$, $L_2$ and $L_\infty$.  To first order these bounds are equal when $d$ is at most
%poly-logarithmic in $T$, but they do differ in second order, or in the constant if $d$ is sub-polynomial in $T$ .
%These bounds are extended to matching bounds in the multi label problem
%where the regret is shown to be $m d / 2 \log (T/ md)$.

\footnotetext[1]{The single dimensional results were known in the information theory literature, and derived using Bayesian mixture methods. \cite{mcmahan12} 
demonstrated their achievability with a \emph{Follow The Regularized Leader (FTRL)\/} gradient method.}
\footnotetext[2]{Results for $m=O(1)$ were known in the information theory literature since \cite{davisson73} and perhaps even before that.  The KT estimator achieves the upper bound
also for $m = \Omega(1)$. Lower bounds were derived in the references cited.}
\footnotetext[3]{The \emph{previous\/} results in the table for $L_1$ and $L_\infty$ are implied from results for $L_2$ in the literature.}
\footnotetext[4]{These upper bounds are derived by extending the derivation from \cite{kakade05} as in Theorem~\ref{the:bayes_ub}. For $L_2$ this was also shown in \cite{foster18}.}

\section{Related Work}
\label{sec:related}

Prior results in both the
machine learning literature (see, e.g. \cite{azoury01, cesa02, littlestone89}) and the information theory literature
(see, e.g., \cite{krichevsky81, merhav95, rissanen84}) illustrate that the performance of the regret (or
\emph{redundancy} in information theory) of the online setting normalized by the number of rounds $T$ meets batch results of convergence rate
at least to first order.  Hence, studying regret in the online setting also implies to the generalization ability of an algorithm.
The setup of a logistic regression problem, whether online (studying regret) or batch (studying convergence rate) is very similar to
the setups of the universal compression problems studied in the information theory literature.  In these problems, the redundancy of algorithms that predict multi label
outcomes in a setup that is equivalent to single dimensional logistic regression with binary features was studied.
It was shown (see, e.g., the seminal work in \cite{rissanen84}, subsequent work in \cite{ds04, orlitsky04, shamir06, spa12}, and references therein) that for these problems, regret of
$\frac{m}{2} \log (T/m)$ is achievable to first order, where $m$ is the number of labels.  However, the concepts presented by \cite{rissanen84} should apply
also to more general $d$ dimensional problems, where $d$ is the number of parameters that affect the label outcome. 
Specifically, in \cite{rissanen84}, central limit arguments, that are also satisfied in the logistic regression setting, were used to 
prove $\frac{d}{2} \log T$ redundancy bounds, when $d = \Theta(1)$.  The subsequent results in \cite{ds04, orlitsky04, shamir06, shamir06a}, however, extended
the redundancy results to $\frac{m}{2} \log (T/m)$, even when $m = T^{1-\e} = o(T)$ (for some small fixed $\e > 0$)
but were more specific to the equivalent of single dimensional logistic regression
with multi $m$ labels.
The machine learning literature considered general online convex optimization, and derived minimax-optimal algorithms for both the linear and strongly convex settings (see, e.g., \cite{abernethy08}), with logarithmic regret in the strongly convex setting.
For weakly convex settings (which generally includes logistic regression), regret rates of $O(d B \sqrt{T})$ have
been shown to be achievable (see, e.g., \cite{zinkevich03}, and references therein), and later rates of $O(B\sqrt{dT})$  (see, e.g., \cite{xiao10}), where $B$ is the radius of the $L_2$ ball defining the allowed values of the parameter $\t_t$, played at round $t$, and the space $\T$ of values of a possible comparator $\t^*$.

To the best of our knowledge, in \cite{kakade05}, a first result suggesting that regret of 
$O(d \log (T/d))$ is achievable for logistic regression, and in
fact for other generalized linear models, was
presented.  Instead of using \emph{gradient methods}, (typically used for this problem)
in which the training algorithm updates the learned parameters taking a step against the
gradient on the loss, the method took from the Bayesian literature to apply \emph{Bayesian Model Averaging} (or \emph{Bayesian mixture\/}) to show a regret 
upper bound (but not a lower bound) that achieves this rate.  In addition, however, the algorithm
pays an additional penalty that depends on the prior selection
as well as on the squared $L_2$ norm of a comparator $\t^*$ (which can be the
loss minimizing parameter in hindsight).  If $\| \t^* \|_2^2$ is larger than the $O(d \log (T/d))$ term,
this penalty term could dominate the bound (depending on the selected prior).  The proof of the bound utilized variational techniques, and
also, in part, resembled some of the central limit arguments used in \cite{rissanen84} to show
upper bounds on redundancy.  The use of Bayesian methods is also justified in the information theory literature (see, e.g., \cite{davisson73, krichevsky81, rissanen84}).  Specifically, \cite{merhav95} showed that a mixture code is as good as the best code in terms of regret
(and thus can be better but not worse than any other type of code). 

\cite{mcmahan12} demonstrated that with binary feature values, using the \emph{Follow-The-Regularized-Leader (FTRL)} methodology
(see, e.g., \cite{hazan12, mcmahan11, rakhlin05, shalev07} and references therein) with a Beta regularizer, $O(\log T)$ regret can be
achieved for the single dimensional problem.  
In the special case of a Beta regularizer with $\alpha=\beta=1/2$, their FRTL algorithm coincides with the well-known
(add-$1/2$) \cite{krichevsky81} (KT) estimator that, in fact, achieves the lower bound on the regret for this problem of $0.5 \log(T)$.
It is interesting to note, however, that the KT method
is derived using a Bayesian mixture with the Dirichlet-$1/2$ prior.  
Thus for the single dimensional case, both the FTRL methodology and the Bayesian mixture one result in the
same estimator.  Unfortunately, this result does not generalize to larger dimensions.

While the lower bound can be achieved for the single dimensional case
for binary features with an FTRL gradient method, \cite{mcmahan12} posed a problem
of what happens in larger dimensions.  The results in \cite{kakade05} hint in the direction of Bayesian methods, but still fall short
of achieving $d/2 \log (T/d)$ regret due to the additional penalty on the prior.  (Although, as we demonstrate, these results with a proper,
perhaps unexpected, 
choice of prior could lead to the desired rates and constants in some cases, but, to the best of our knowledge, such a result was not reported in the literature.)
A series of papers \cite{bach10, bach13, bach14} studied the convergence rate of gradient methods for logistic regression, and concluded,
that while logistic loss is not globally strongly convex, it can, depending on the actual data, locally exhibit strong convexity (referred to 
as the \emph{self-concordance} property).  Then, gradient methods can achieve convergence rate of $O(1/\lambda T)$, where $\lambda$ is
the smallest eigenvalue of the Hessian at the global optimum.  This implies that gradient methods can, in many case, achieve logarithmic
regret, but there do exist situations where gradient methods fail
to achieve $O(d \log(T/d))$ regret (when $\lambda$ is small).

%The results in \cite{bach14} demonstrate that even with only binary features, gradient methods may fail to achieve logarithmic regret.
%Consider a sparse problem when two different features are
%observed one in one set of $O(T)$ examples and the other in another set of $O(T)$ examples.  Suppose that for all these examples only
%the same single label value is observed. 
%A gradient method can learn large $O(\log T)$ parameters for both features.  When it tries to generalize on an example
%that consists of both features, it could predict a 	``confident'' prediction to a level that is beyond the confidence imposed by the number of observed
%examples.  Situations like these can lead to very small eigenvalues in the Hessian, which will result in convergence rates greater than 
%$O(1/T)$.  A different example is the following.  Consider an adversarial feature sequence, where a feature co-occurs alternately with
%other features that learned parameters that are extreme in opposite directions.  However, the labels are always opposite to the values
%learned for each of the other parameters.
%This situation can flatten the objective w.r.t.\ this feature, yielding low eigenvalues
%representing this feature, overall increasing the convergence rate of an algorithm.  Making a decision on the value
%of $\t_t$ for the next round based on such uncertain training data thus leads to higher regret.  This implies that Bayesian
%methods, that do not make ``confident'' decisions based on uncertain data, may thus be superior to gradient methods that do.

\cite{hazan14} studied the problem, in which Bayesian methods are not possible to apply directly, where the feature values are unknown
when playing $\t_t$ at round $t$, and are only revealed later, together with the label.  Bayesian methods do condition the predicted label probability on
the observed feature values, and if such are not available, they would require also mixing on the feature values.
It was shown, that in this setting, which is more difficult
to the algorithm, regret of $O(B^3 T^{1/3})$ is achieved for the single dimensional problem where only
$\Omega(B^{2/3} T^{1/3})$ is possible, and $\Omega (\sqrt{BT})$ is only possible
for any larger dimensions even for $d=2$.
%These results, however, are not restricting the feature values in any way.  In many practical problems though, the
%feature values are either binary or can be restricted to some practical range.  In such cases, better rates than those obtained may be possible.   

\cite{foster18} separated the problem posed in \cite{mcmahan12} to the case considered by \cite{hazan14}, where the algorithm plays
$\t_t$ with no knowledge of the feature values $x_t$, which is referred to as a \emph{proper} setting, and to the \emph{mixable} setting
where the feature values are revealed to the algorithm prior to generating a prediction, referred to as the \emph{improper} setting.
Using Bayesian model averaging with a uniform prior with an approach that resembles
that in \cite{kakade05}, an upper bound of $O(m d \log (T/md))$ was shown for the multi label 
$d$-dimensional (with $d$ distinct features) logistic regression problem, where $m$ is the number of distinct labels,
under $L_2$ constraints on $\t^*$. 
A lower bound of $\Omega(d)$ was shown for the binary labels / binary features setting under the constraints that
$B = \Omega ( \sqrt{d} \log T)$.
The upper bound matches the logarithmic order of the bound expected from the information theory problems, but not
the constant, and the lower bound is lower in order.
%With an appropriate (unexpected) selection of the prior in \cite{kakade05}, a matching upper bound can be attained.

The results summarized above suggest that there are, in fact, two different sets of online logistic problems considered.
In the first, the features $x_t$ are revealed prior to playing $\t_t$ or to generating a prediction,
and in the second, $\t_t$ is played before the feature values are revealed.
The first problem allows the use of Bayesian methods, while the second will require such methods to also mix over the
unseen $x_t$.
For the first problem, logarithmic regret is possible for low dimensionalities, whereas for the second extreme case, it is not in many settings,
even in the single dimensional problem.
In this paper, we give a comprehensive characterization of the regret behavior for the first problem, including the asymptotic
regime, where $d$ is allowed to grow with $T$.  The lower bounds we derive apply to any case, including the second problem, but
the upper bounds are specific to the first one.

\section{Problem Formulation, Notation and Definitions}
\label{sec:notedef}

We consider online convex optimization over
a series of rounds $t \in \left \{1, 2, \ldots, T \right \}$ as in \citep{mcmahan14}
(see also \cite{boyd04, rockafellar97, shalev12}).
Each round $t$, a $d$-dimensional \emph{example} feature vector $x_t \dfn \{x_{t,1}, x_{t,2}, \ldots, x_{t,d}\} \in \X$ and a label $y_t  \in \Y$
are observed.  For the binary labels, we use $\Y = \{-1, 1\}$.
We assume, without loss of generality, that $|x_{t,i}| \leq 1$, as features can be normalized.
%(In Section~\ref{sec:proper-regret}, it will be convenient to have $x_{t,i} \geq 0$, but we do not make this assumption otherwise.)
We denote a subsequence up to time $t$ by $x^t \dfn \{x_1, x_2, \ldots, x_t\}$.  For the example/label
pair, we also use $S_t \dfn \{(x_s, y_s)\}_{s=1}^{t}$.
Capital letters denote random variables.
A learning algorithm $\A$ is a function that, given a sequence $S_{t-1}$,
an example $x_t$, and an
arbitrary label $y \in \Y$, returns at round $t$ a probability for the label
\be
 \A (S_{t-1}, x_t, y) \dfn P [Y_t = y | X_t = x_t, S_{t-1}].
\ee
To produce a prediction, an algorithm may play a weight vector $\t_t \in \T$, or perform a Bayesian mixture
over $\t \in \T_m \subseteq \TT \subseteq \Rspace^d$.  For a given model $\t$, the probability of a label for example $x$ is given by
%\be
%\label{eq:sigmoid_prob}
 $p \left (y | x, \t \right ) \dfn \frac{1}{1 + \exp(-y \cdot x^T \t)}$.
%\ee
The loss at $t$ for model $\t$ is
%\be
% \label{eq:round_loss}
$ \ell (\t, x_t, y_t) \dfn -\log p(y_t | x_t, \t) = \log \left ( 1 + \exp(-y_t \cdot x_t^T \t) \right )$,
%\ee
where it will sometimes be convenient to use the dot product $z \dfn x^T \t$.  
Similarly, the loss of $\A$ at $t$ is $\ell (\A, x_t, y_t ) \dfn -\log [ \A (S_{t-1}, x_t, y_t) ]$.
The total loss for model $\t$ on sequence $S_T$ is $L (\t, S_T) \dfn \sum_{t=1}^T \ell(\t, x_t, y_t)$.
Similarly, $L(\A, S_T) \dfn \sum_{t=1}^T \ell(\A, x_t, y_t)$.

%The total loss incurred by an algorithm $A$ that plays $\t_t$ is given by
%\be
% \label{eq:total_loss}
% L_T \left (A, z^T\right ) \dfn \sum_{t=1}^T \ell (\t_t, z_t).
%\ee
%A loss for an algorithm that plays the same $\t$ in every round is denoted by $L_T(\t, y^T, x^T) \dfn L_T (\t, z^T) \dfn \sum_t \ell(\t, z_t)$.
The \emph{regret} of $\A$ for a given example/label pair sequence 
$S_T$ relative to a comparator model $\t^* \in \T  \dfn \{\t : \|\t \|_\rho \leq B \}$, where $B$ constrains the norm of $\t^*$,  is defined as
\be
 \label{eq:regret_comparator_def}
 \R
  \left (\A, S_T, \t^* \right ) \dfn L \left (\A, S_T \right ) - L \left ( \t^*, S_T \right ).
\ee
We limit the comparator such that $\t^* \in \T$, and consider the different cases
where $\rho \in \{ 1,2, \infty\}$.  It is reasonable to assume that $B = \gamma \log T$ for some $\gamma > 0$.
%For algorithms that play $\t_t$, we limit $\t_t \in \T$ as well.
\emph{Bayesian mixture} algorithms could have support in $\T_m \subseteq \TT$ where $\TT \supseteq \T$.
The regret of $\A$ relative to the best comparator is given by
\be
 \label{eq:regret_def}
 \R  \left (\A, S_T \right ) \dfn \sup_{\t^* \in \T} \R \left (\A, S_T, \t^* \right ).
\ee

A mixture algorithm that may rely on the values of $x_t$ in its predictions of $y_t$, predicts
%a label sequence probability
\be
 \label{eq:mixture_seq_probability}
 p(y^t | x^t) \dfn \int_{\t \in \T_m} p \left (y^t | x^t, \t \right ) \cdot p_0(\t) d\t =
 \int_{\t \in \T_m} \prod_{\tau = 1}^t p \left (y_\tau | x_\tau, \t \right ) \cdot p_0 ( \t )d \t
\ee
where $p_0 (\t)$ is some initial \emph{prior} on the distribution of the parameter vector $\t$, and $\T_m \subseteq \TT$ is the support of
the mixture, which may be different form $\T$.  The probability \eref{eq:mixture_seq_probability}
assigned to $y^t$ can also be expressed as a set of equations that sequentially
update a \emph{posterior} distribution over $\t$ at round $t$ from the prior at $t$, which is the posterior at $t-1$, i.e.,
\be
 \label{eq:mixture_posterior}
 p(\t | S_t ) = \frac{\prod_{\tau = 1}^t p \left (y_\tau | x_\tau, \t \right ) \cdot p_0 ( \t )}
 {\int_\t \prod_{\tau = 1}^t p \left (y_\tau | x_\tau, \t \right ) \cdot p_0 ( \t )d \t}
 \dfn \frac{p(\t, y^t | x^t)}{p(y^t | x^t)}.
\ee
The prediction of $y_t$ is then given by
\be
 \label{eq:mixture_label_prediction}
 p (y_t | x_t, S_{t-1} ) = \int_{\t} p(y_t | x_t, \t) \cdot p(\t | S_{t-1} ) d \t.
\ee
As seen in \eref{eq:mixture_label_prediction}, the prediction is also conditioned on the feature values (example) vector $x_t$.
The prior distribution $p_0(\t)$ is shown to be continuous in \eref{eq:mixture_seq_probability}-\eref{eq:mixture_label_prediction}.
However, $\T_m$ can be set to be a discrete set, and then
\eref{eq:mixture_seq_probability} can be re-rewritten as
\be
\label{eq:discrete_mixture_seq_probability}
 p(y^t | x^t) \dfn  \sum_{\t \in \T_m} p \left (y^t | x^t, \t \right ) \cdot p_0(\t) =
 \sum_{\t \in \T_m} \prod_{\tau = 1}^t p \left (y_\tau | x_\tau, \t \right ) \cdot p_0 ( \t ).
\ee

\section{Useful Methods}
\label{sec:background}

\subsection{Lower Bounds on Regret - The Redundancy Capacity Theorem}
A lower bound on regret is meaningful only when stated in terms of existence of a sequence $S_T$ for every possible
algorithm, for which the regret is at least the lower bound.
\cite{davisson73} formulated such a notion for universal compression redundancy as the \emph{redundancy-capacity theorem}
which showed that the redundancy
(or regret) can be lower bounded by the mutual information $I(\T; \S_T)$ between the parameter and the observed data sequence,
induced by the prior on $\T$ of a mixture model.
A specific interesting case is when the prior is uniform on a discrete subset $\T_m \subseteq \T$ of
the parameter space, and the elements in $\T_m$ are \emph{distinguishable\/} by the observation $\S_T$, i.e., observing
$S_T$ is sufficient to determine which $\t \in \T_m$ generated $\S_T$ with error probability $P_e \rightarrow 0$ as ${T \rightarrow \infty}$.
This case leads to a weaker lower bound than the bounds described in \cite{davisson73} and subsequent works, but is
sufficient for showing redundancy bounds in many cases (see, e.g., \cite{merhav95, shamir06}), and also for
regret bounds for our problem.  We frame this result to regret, and prove it by mirroring the part of the derivation
in \cite{davisson73} that is sufficient for the result we need, but described in terms that apply to the regret problem.
We next state the theorem, which is proved in Appendix~\ref{sec:red_cap_proof}.

\begin{theorem}{Distinguishable Grid Regret} (adapted from \cite{davisson73}):
\label{the:red-cap}
Let $T \rightarrow \infty$.
Let $\T_m \subseteq \T$ be a set of $M \rightarrow \infty$ distinct values of $\t$.  Draw $\Phi \in \T_m$ with a uniform prior, and generate
$\S_T$ from the distribution determined by $\Phi$.
Let $\hat{\Phi} \dfn f (\S_T)$ be some estimator of $\Phi \in \T_m$ from the observed $\S_T$.  Then, if 
$P_e \dfn P(\hat{\Phi} \neq \Phi ) \rightarrow 0$, the regret of any algorithm $\A$ for the worst sequence $S_T$ is lower bounded
by
\be
 \label{eq:red-cap-regret}
 \sup_{S_T} \R  \left (\A, S_T \right ) \geq (1 + o(1)) \log M.
\ee
Similarly, for a fixed $x^{*T}$, if we draw $Y^T$ instead of $S_T$ and the conditions above hold, \eref{eq:red-cap-regret} also holds.
\end{theorem}

\subsection{Variational Approach for Upper Bounds}
\label{subsec:var}

Upper bounds can be obtained by showing Bayesian methods that can achieve low regret and bounding their regret.  For simplicity,
one can select priors $p_0(\cdot)$ with a diagonal covariance.  \cite{kakade05} selected a normal  prior, whereas \cite{foster18} used a uniform one.
We follow \cite{kakade05} and manipulate their approach to obtain tighter bounds for $L_2$ and $L_\infty$, and then use the method of
grids with a uniform discrete prior combined with their method to derive an $L_1$ bound.
We first describe their approach.
Define a distribution $Q(\t)$ on $\T_q \subseteq \T_m$ where $E_q(\t) = \t^*$, and $E_q[(\t_i - \t_i^*)(\t_j - \t_j^*)]  \leq \eta^2_q \cdot \delta(i-j)$,
where $\delta(n) = 1$ if $n=0$ and is $0$, otherwise,
i.e, diagonal covariance matrix, where $\eta_q^2$ is an upper bound on elements of the diagonal.
(Note that $\T_q$ can be a subset of $\T$, but does not have to, and in fact, is not for the normal prior).
Let $D(Q||p_0)$ be the KL-divergence between $Q$ and
$p_0$. Then the following theorem holds.
\begin{theorem}
\label{the:variational} \cite{kakade05}:
The regret of a Bayesian algorithm $\A^*$ with prior $p_0$ for sequence $S_T$ and comparator $\t^*$ is upper bounded by
\be
 \label{eq:var_ub}
 \R \left (\A^*, S_T, \t^* \right ) \leq  D(Q||p_0) + \frac{dT}{8} \eta^2_q
\ee
\end{theorem}
The proof of Theorem~\ref{the:variational} is in \cite{kakade05}, but needs to be modified a bit because they restricted
$\|x_t\|_2 \leq 1$, while we assume $|x_{t,i}| \leq 1$ ($\|x_t\|_\infty \leq 1$).  We rely on their proof, except where it needs to be modified.
The proof is in Appendix~\ref{ap:var}.

\section{Regret Lower Bounds}
\label{sec:lb}
We now use Theorem~\ref{the:red-cap} to derive lower bounds for the binary label case.  A lower bound for the multi label case is given
in Appendix~\ref{ap:multi}.  We first define
\be
 \label{eq:gamma_def}
 \gamma \dfn \min \left \{ \frac{B}{\log T}, \left \{ \alpha : \alpha \cdot \min[T^{1-\e}, d] = T^{1-\e} \right \} \right \}
\ee
as the effective count of $\log T$ units in $B$ (where if $d$ is very large, we will only consider a clipped portion of $\T$ for $\T_m$
to ensure distinguishability.  This
will guarantee that $\gamma \min [T^{1-\e}, d] = o(T)$).
The lower bounds are stated in the following theorem.
\begin{theorem}
\label{th:binary_lb}
Fix an arbitrary $\e > 0$, let $T\rightarrow \infty$.  Then, for every algorithm $\A$ there exists a sequence $S_T$, for
which the regret is lower bounded by
\be
\label{eq:regret_lbs}
\R (\A, S_T) \geq \left \{ \begin{array}{ll}
(1 - o(1)) \frac{d}{2} \log \frac{T}{d}; & \text{for } d = O(1), \\
%(1 - o(1)) \frac{d}{2} \log \frac{4 B^2 T}{d}; & \text{for } \|\t^* \|_\infty \leq B \leq \log T, \\
(1 - o(1)) \frac{d}{2} \log \frac{4 \gamma T}{d}; & \text{for } \|\t^* \|_\infty \leq B \text{ and } d < \frac{4}{e} \gamma T^{1-\e}, \\
(1 - o(1)) \frac{2}{e} \gamma T^{1-\e}; & \text{for } \|\t^* \|_\infty \leq B \text{ and } d \geq \frac{4}{e} \gamma T^{1-\e}, \\
(1 - o(1)) \frac{d}{2} \log \frac{2 \pi e \gamma T }{d^2}; & \text{for } \|\t^* \|_2 \leq B \text{ and } d <
  \sqrt{\frac{2\pi}{e} \gamma T^{1-\e}}, \\
(1 - o(1)) \sqrt{\frac{2\pi}{e} \gamma T^{1-\e}}; & \text{for } \|\t^* \|_2 \leq B \text{ and } d \geq
  \sqrt{\frac{2\pi}{e} \gamma T^{1-\e}}, \\
(1 - o(1)) \frac{d}{2} \log \frac{4e^2 \gamma T}{d^3}; & \text{for } \|\t^* \|_1 \leq B \text{ and } d <
  \left( \frac{4 \gamma T^{1-\e}}{e} \right )^{1/3}, \\
 (1 - o(1)) \frac{3}{2} \left( \frac{4 \gamma T^{1-\e}}{e} \right )^{1/3}; & \text{for } \|\t^* \|_1 \leq B \text{ and } d \geq
  \left( \frac{4 \gamma T^{1-\e}}{e} \right )^{1/3}.
 \end{array}  
 \right .
\ee
\end{theorem}

Theorem~\ref{th:binary_lb} shows that for small $d$ each feature/dimension contributes $0.5 \log (T/d)$ to the worst case regret.
Generally, for $L_\infty$ each parameter costs $0.5 \log ( \gamma T / d)$.  For
$L_2$ there is a reduction inside the logarithm by a factor of $d$, and an additional similar reduction is observed between
$L_2$ and $L_1$.  These relations are expected, because the relations between the bounds reflect the logarithm of the ratio between
the respective volumes of the parameter spaces, dictated by the constraints.  The greater the volume, the harder the algorithm
has to work to match the best comparator, and the larger the regret penalty it pays.  This is similar to observations in the information
theory literature, which tie the redundancy to the richness of the class.  The dependence on $B$ is through $\gamma$.  Each interval
of $\log T$ consists of roughly $\sqrt{T/d}$ distinguishable parameters.  Hence, the ratio between $B$ and $\log T$ dictates how
many parameter regions are in an interval of diameter $2B$.  Thus this ratio, represented by $\gamma$, dominates the effect of $B$,
which is normally, in practice, negligible relative to the effects of $T$ and $d$. (In practice, we would normally limit $\t$ to some reasonable
range, which is usually $O(1)$.  However, theoretically, $\gamma$ can be larger, in which case it does influence the bound.)  If $B$
is too large, the effective $\gamma$ in \eref{eq:gamma_def} guarantees that $\gamma d = o(T)$, and if $d = \Omega(T)$, it guarantees
this with respect to the largest value $T^{1-\e}$ used for the bound.
%Note also that, in fact, $\gamma$ could also multiply some $\kappa \dfn (\log T) / (\log (T/d))$ in all regions of the bound, as we can, in fact,
%partition $B$ into regions of $\log (T/d)$.  However, we omit $\kappa$ for slightly looser bounds in some cases, for the sake
%of simplicity.

An interesting behavior is observed for all cases $L_\infty$, $L_2$ and $L_1$.  When we reach
$d = O(T)$, $d = O(\sqrt{T})$ and $d = O(T^{1/3})$, respectively,
a threshold phenomenon happens, and the bound becomes constant for every greater $d$.  It is not clear how much of this
is a result of the bounding techniques and how much is real.  However, as we see in the upper bounds for $L_1$ in the following section,
there exist Bayesian algorithms that achieve $o(d)$ regret for $L_1$ constraints.  We also observe a decrease in rate
in the upper bounds for $L_2$ at $d = O(T)$.
Together, these results imply, that there are, in fact, cases in which
the regret does not grow linearly with $d$ for large enough $d$.
The $L_\infty$ bounds demonstrate that there are situations in which we are guaranteed regret rates of $\Omega(\sqrt{T} \log T)$.
In fact, the regret could even be linear with $d$ up to $d = T^{1-\e} = o(T)$.

To prove the first region of Theorem~\ref{th:binary_lb}, we partition $x^{*T}$ into $d$ separate equal length
segments, where in each segment only
one component  $x^*_{t,i}$ of $x^*_t$ is $1$ and the rest are $0$.  
This transforms the problem to a standard universal compression problem in $d$ different segments, in each a single parameter is to be
estimated.  In each segment,
we now have a grid of $\sqrt{T/d}^{1-\e}$ points, which are spaced (in the space of label probability they induce)
at a distance $\sqrt{d/T}^{1-\e}$ from one another.
The total grid $\Psi$ is the power set of the individual grids over the segments.
Large deviation typical sets analysis (see \cite{cover06}) with the union bound over the segments is used to show that each
of these points is distinguishable from the others. 
Finally, applying Theorem~\ref{the:red-cap} with 
a fixed $x^{*T}$ gives the lower bound.
For diminishing large deviation exponent to dominate over the union bound, we need to use $d  = o(T)$.  For larger values
of $d$, we use a grid that varies only in the first $T^{1-\e}$ components of $\t$, and apply the resulting bound.

For the remaining regions, we fix the first component of the parameter at the maximal point $\t_1 = B$.  Then, $x_{t,1}$ would
scale it by factors of $i/\gamma$, where $i$ is an integer, taking all values from $-\gamma$ to $\gamma$.  This will induce $2\gamma+1$
priors, that make $(2\gamma + 1 )$ distinct distinguishable regions of a second nonzero component $\t_i$ of $\t$
that occurs in the same examples as $\t_1$.
We partition each of now $d-1$ segments, where in each of these segments a different component $x^*_{t,i}$  is $1$ for $i > 1$, while
the remaining ones are $0$, to $(2\gamma + 1 )$ subsegments corresponding to the different values of $x_{t,1}$.
We show that the points on the grid, now constructed as a power set of
$d-1$ grids of $(2\gamma + 1) \sqrt{T/d\gamma}^{1-\e}$ points, are, again, distinguishable with the fixed $x^{*T}$.  Using Theorem~\ref{the:red-cap},
the logarithm of the cardinality of the power grid lower bounds the regret.  However, for $L_2$ and $L_1$, only the components of
$\t \in \Psi$ which satisfy the constraints are included in the grid.  This reduces the bounds, and leads to a threshold point, in which
the lower bounds become useless for the value of $d$, since the remaining space no longer contains parameters for which 
all components of $\t$ are nonzero.  We can thus use the value of $d$, which is lower, but achieves the largest bound.  This leads
to regions 3, 5, and 7 in the bound.  The proof of Theorem~\ref{th:binary_lb} is presented in Appendix~\ref{ap:lb_proof}.

\section{Regret Upper Bounds for Bayesian Methods}
\label{sec:ub}
Theorem~\ref{the:variational} allows us to prove the following two theorems:
\begin{theorem}
\label{the:bayes_ub}
There exist Bayesian algorithms $\A^*$ that for every sequence $S_T$ and comparator $\t^* \in \T$, achieve regret
\be
 \label{eq:regret_ub}
 \R (\A^*, S_T, \t^* ) \leq \left \{ \begin{array}{ll}
 (1 + o(1)) \cdot \frac{d}{2} \log \left (\frac{B^2 T e}{4} + e \right ); & \text{for } \|\t^* \|_\infty \leq B, \\
 (1 + o(1)) \cdot \frac{d}{2} \log \left ( \frac{B^2 T e}{4d} + e \right ); & \text{for } \|\t^* \|_2 \leq B, \\
 (1 + o(1)) \cdot \frac{d}{2} \log \left (\frac{B^2 T e^3}{4d^2} \right ); & \text{for } \|\t^* \|_1 \leq B \text{ and } d = o(B\sqrt{T}), \\
 (1 + o(1)) \cdot \left ( \frac{d}{2} \log(4e) + \frac{B\sqrt{T}}{2} \right ); & \text{for } \|\t^* \|_1 \leq B \text{ and } d = \T(B\sqrt{T}), \\
 (1 + o(1)) \cdot  \left ( \frac{d}{2} + \sqrt{2 d B \sqrt{T}} \right ); & \text{for } \|\t^* \|_1 \leq B \text{ and } d = \Omega (B\sqrt{T})
 \end{array}  
 \right .
\ee
\end{theorem}

\begin{theorem}
\label{the:bayes_ub1}
Let $d = \Omega(B\sqrt{T})$.
Then, there exists a Bayesian algorithm $\A^*$ that for every sequence $S_T$ and comparator $\t^* \in \T$ with $\|\t^*\|_1 \leq B$,
achieves regret
\be
 \label{eq:regret_ub1}
 \R (\A^*, S_T, \t^* ) = O \left ( T^{1/5} d^{3/5} B^{2/5} \right ) = o(d).
\ee
\end{theorem}

Theorem~\ref{the:bayes_ub} shows that regret logarithmic with $T$ and linear with $d$ is achievable in all three cases.  The bounds asymptotically
differ only by a factor of $d$ inside the logarithm.  The wider allowable range of $\t^*$ gives an upper bound where $T$ in the logarithmic
term is not normalized by $d$.  The smaller comparator region, where $\t^*$ norms are restricted by $L_2$, reduces the logarithmic
cost from $\log T$ to $\log (T/d)$.  A similar reduction is achieved from $L_2$ to $L_1$.
Both $L_2$ and $L_1$ have interesting threshold behavior, which matches the behavior with the lower bounds.  For $L_1$,
as long as  $d = o(B\sqrt{T})$, we observe regret
linear in $d$ and logarithmic in $T$.  For larger dimensions, though, we observe only linear behavior in $d$, without the logarithmic
terms.  Furthermore, if we tighten the bounds further, Theorem~\ref{the:bayes_ub1} shows that even sub-linear behavior 
in $d$ is possible
%in this case
($o(B\sqrt{dT})$ for $d = o(B^6 T^3)$).  For $L_2$, the transition from $O(d \log T)$ to $O(d)$ occurs at $d = O(B^2 T)$.

The bounds in Theorem~\ref{the:bayes_ub} are derived for our setting of $\|x_t\|_\infty \leq 1$.  \cite{kakade05} and others considered
the setting where $\|x_t\|_2 \leq 1$.  In this setting, all the bounds in \eref{eq:regret_ub} will have an additional factor of $d$ in the
denominator of the logarithmic term.  This means that the transition from $O(d \log T)$ to $O(d)$ rate (and for $L_1$ to $o(d)$) will now
occur at $d = O(T^{1/3})$, $d = O(\sqrt{T})$ and $d = O(T)$ for $L_1$, $L_2$, and $L_\infty$, respectively.  (Such a transition will now also
happen for $L_\infty$.)

Unlike Theorem~\ref{th:binary_lb}, $B$ is present in the upper bounds instead of $\gamma$.  For distinguishability
on an individual feature, each region of $\log T$ in a single dimension consists of only $\sqrt{T}$ distinguishable points, and not
$\sqrt{T} \log T$.  However, when dimensions are mixed through the dot product, it is hard to disentangle the dimensions.  This leads to
the difference between the lower and the upper bounds.

The proof of Theorems~\ref{the:bayes_ub} and~\ref{the:bayes_ub1} is based on Theorem~\ref{the:variational}.  For $L_\infty$ and $L_2$,
we use a Gaussian prior with a Gaussian $Q(\cdot)$.  We derive the bounds on the terms of \eref{eq:var_ub}, find the value of the parameter
that gives the smallest bound and apply it.  We then find the variance of the prior that gives the tightest bound, and apply it.  For $L_1$,
we construct a grid whose points are assigned a uniform probability.  We construct $Q(\cdot)$ as a Bernoulli distribution in each dimension,
giving nonzero probability only to the surrounding neighbors of the $i$th component $\t^*_i$  of $\t^*$, but ensuring that $\t^*$ is
the expectation of $Q(\cdot)$.  Then, in the same manner, we upper bound the terms of \eref{eq:var_ub}, and optimize the free parameter
for the tightest bound.  Finally, since a threshold occurs, where the bound becomes useless, we use a union bound on lower dimensions
of the parameter space. We find the dimension that gives the maximal element of the sum over all dimensions, and use it to upper
bound all dimensions.
Applying this method more tightly gives some tedious algebra, but yields a bound of $o(d)$ for the upper region
of the $L_1$ constraint problem.  
The proof of both theorems is presented in Appendix~\ref{ap:ub_proof}.

\section{Conclusions}
\label{sec:conclusions}
We studied logistic regression regret, and derived lower and upper bounds for settings constrained by the norm of 
a comparator.  We presented a comprehensive characterization of the regret for the different settings, including the asymptotic 
behavior in the dimensionality. Adapting a methodology from the universal compression literature, we derived
lower bounds on the regret, showing initial logarithmic in $T$, linear in $d$ regret, with rates whose growth slows with larger dimensions
of the feature space.  Matching upper bounds confirm the general behavior of the lower bounds.  Specifically, we demonstrated
that under $L_1$ constraints, for large enough $d$, regret becomes sub-linear in $d$, and for $L_2$ constraints, it drops from
linear in $d$ and logarithmic in $T$ to just linear in $d$.  On the negative side, under $L_\infty$ constraints, regrets of 
$\Omega(\sqrt{T} \log T)$ are guaranteed for $d = \Omega(\sqrt{T})$.

%\nocite{*}

\acks{The author acknowledges Matt Streeter for many helpful discussions and comments.}

\vskip 0.2in
\bibliography{logistic-regret}

\appendix
\section{Proof of Theorem~\ref{the:red-cap}}
\label{sec:red_cap_proof}

The following lemma is needed to prove Theorem~\ref{the:red-cap}.
\begin{lemma}
\label{l:prior-regret}
Let $\t^* \in \T_m$ be a parameter in the support of a Bayesian algorithm $\A^*$ that predicts as described in
\eref{eq:mixture_seq_probability}-\eref{eq:mixture_label_prediction} or in \eref{eq:discrete_mixture_seq_probability}.  Then, 
\be
 \R (\A^*, S_T, \t^* ) = \log p_{\A^*} (\t^* | S_T) - \log p_0 (\t^*).
\ee
\end{lemma}
\begin{proof}
By definition
\bea
 \nonumber
 \R (\A^*, S_T, \t^* ) &=&
 \log p(y^T | x^T, \t^*) - \log \int_{\t \in \T_m} p(y^T | x^T, \t) p_0(\t) d \t \\
 \label{eq:prior-regret-proof}
 &=&
 \log \frac{p(y^T | x^T, \t^*) p_0(\t^*)}{p_{\A^*}(y^T | x^T) p_0(\t^*)} =
 \log p_{\A^*} ( \t^* | S_T) - \log p_0 (\t^*)
\eea
where the equalities are obtained by multiplication and division by $p_0(\t^*)$, which by the conditions of the lemma ($\t^* \in \T_m$) is greater than $0$, and by identifying the
posterior derived from $\A^*$.
\end{proof}

\begin{proof}{of Theorem~\ref{the:red-cap}:}
Let $\A^*$ be a Bayesian algorithm as defined in \eref{eq:discrete_mixture_seq_probability} on a discrete $\T_m$.  Let $p_0(\t)$ now
be uniform with $p_0(\t) = 1/M, \forall \t \in \T_m$, and $\|\T_m\| = M$ (where $M$ can be a function of $T$).  Now,
\bea
 \nonumber
 \sup_{S_T} \R  \left (\A, S_T \right ) &=& \sup_{S_T} \sup_{\t^* \in \T} \R  \left (\A, S_T, \t^* \right )
 \stackrel{(a)}{\geq} \sup_{\t^* \in \T_m} \sup_{S_T} \left [ L(\A, S_T) - L(\t^*, S_T) \right ] \\
 \nonumber
 &\stackrel{(b)}{\geq}& \sup_{\t^* \in \T_m} E_{\S_T | \t^*} \left [ 
 \log p(Y^T | X^T, \t^*) - \log \A ( S_T, X^T, Y^T) 
 \right ]  \\
 \nonumber
 &\stackrel{(c)}{\geq}& 
 E \log p(Y^T |X^T, \Phi) - E_{\S_T} E_{\T_m|\S_T} \log p_{\A^*} (Y^T | X^T, \S_T) \\
 \label{eq:th_red_cap_proof}
 &=&
 E [ \R ( \A^*, \S_T, \Phi) ]
 \stackrel{(d)}{=} 
 E \log p_{\A^*} (\Phi | \S_t) - E \log p_0 (\Phi)
\eea
Step $(a)$ follows for exchanging order of supremums, and shrinking $\T$ to $\T_m$.  Substituting loss definitions, and 
lower bounding the supremum on $S_T$ by an expectation over $S_T$ conditioned on $\t^*$ leads to $(b)$.  Step $(c)$ follows
from lower bounding the supremum on $\T_m$ by expectation of $\T_m$ w.r.t.\ $p_0(\t)$.  This yields expectation w.r.t. $\t^*$ and $S_T$
for the left term.  Performing this expectation on the right term implies expectation on $Y^T$ with a distribution that is the one assigned
to $Y^T$ by $\A^*$.  This negative logarithm is minimized with predictions given from $\A^*$, leading to the right term in step $(c)$, which,
similarly to the left term, performs the expectation on both $\S_T$ and $\T_m$.  The resulting expression is the expectation of the regret
of $\A^*$ on $\T_m$.  Applying Lemma~\ref{l:prior-regret}, gives $(d)$.  By the uniform construction of $p_0(\cdot)$, the right term is
$\log M$.  The left term can be bounded using $P_e$ (see, e.g., Fano's inequality, \cite{cover06})
\be
 \label{eq:fano_posterior}
 -E \log p_{\A^*} (\Phi | \S_t) \leq h_2 (P_e) + P_e \log (M - 1) \leq 1 + P_e \cdot \log M,
\ee 
where $h_2(p) \dfn -p\log p - (1-p)\log (1-p)$ is the binary entropy function.  This is proved by expectation over $\Phi$, and then, 
breaking the events of the value of $\hat{\Phi}$ into the event $\hat{\Phi} = \Phi$ and $\hat{\Phi} \neq \Phi$, and then hierarchically
separating the latter into the $M-1$ different
possible values of $\hat{\Phi}$, upper bounding the conditional entropy on $\hat{\Phi} \neq \Phi$ by that of a uniform distribution.
Combining both terms of \eref{eq:th_red_cap_proof}, given $M \rightarrow \infty$ and $P_e \rightarrow 0$, concludes the proof of the
first statement of 
Theroem~\ref{the:red-cap}.  The second statement follows the exact same derivation conditioned on a fixed $x^{*T}$, after lower
bounding the supremum over $\S_T$ by
that for this fixed $x^{*T}$.
\end{proof}

\section{Proof of Theorem~\ref{the:variational}}
\label{ap:var}

\begin{proof}{of Theorem~\ref{the:variational}:}
Let $L ( Q, S_T ) \dfn \int_{\t \in \T_q} L(\t, S_T) Q(\t) d\t$. (Similarly, if $\T_q$ is discrete, the integral is replaced by a sum).  
Then, for a Bayesian algorithm $\A^*$
with prior distribution $p_0(\cdot)$ on a logistic regression model,
\be
 \label{eq:regret_var}
 \R \left (\A^*, S_T, \t^* \right ) =
  \underbrace{ L \left (\A^*, S_T \right ) - L \left (Q, S_T \right )}_{\leq D(Q||p_0) } +
   \underbrace{L \left (Q, S_T \right ) - L \left ( \t^*, S_T \right )}_{\leq d\cdot T \cdot \eta^2_q / 8}
\ee
The first term is bounded in Lemma~2.1 in \cite{kakade05}.  For the second term, 
recall the dot product $z_t \dfn x_t^T \t$. Then,
for some $\t$, round $t$, and example/label pair
$\{x_t, y_t\}$, define $f (z) \dfn \ell(\t, x_t,y_t)$ as the per example/label loss, which can be expressed as just a function of $z$.  Then,
using Taylor expansion,
\be
 E_q[ f(z)] - f(z^*) =  f'(z^*) \cdot 0 + E_q \left [ f''(\xi(z)) \frac{(z-z^*)^2}{2} \right ]  \leq \frac{d \cdot \eta^2_q }{8}
\ee
where the first term is $0$ by definition of the expectation w.r.t.\ $Q(\cdot)$. Then $\xi(z)$ is some point between $z^*$ and $z$ for
which equality is satisfied for the second order Taylor approximation.  The second derivative of logarithmic loss w.r.t.\ the dot product
for a diagonal term is $x_{t,i}^2 p (1-p)$ for some probability $p$, and thus upper bounded by $1/4$.  By construction of a
diagonal covariance matrix, $\sigma_z^2 \leq \eta_q^2 x^T x \leq d \eta_q^2$ .  Since this bound is added on $T$ examples,
this term is multiplied by $T$.  This concludes the proof.
\end{proof}

\section{Proof of Theorem~\ref{th:binary_lb}:}
\label{ap:lb_proof}

\begin{proof}{of Theorem~\ref{th:binary_lb}:}
We construct grids $\grid = \T_m$ of points $\psi_i$ for all dimensions $i$
in the parameter space, fix $x^{*T}$, show that the points are distinguishable
given $x^{*T}$, and use the 
version of Theorem~\ref{the:red-cap} for a fixed $x^{*T}$ to lower bound the regret by the logarithm of the cardinality of the grid $\grid$.
For the first region, partition $x^{*T}$ into $d$ separate length $T/d$
segments.  For segment $i, i = 1,2,\ldots, d$;  $x^*_{t,i} = 1$ and all $x_{t,j} = 0$ for $j\neq i$.
The grid $\grid$ is a power set of grids for dimension $i$, for all $i$.  Define
\be
 \label{eq:grid_prob}
 p_{i,j} \dfn \frac{1}{1 + \exp(-\psi_{i,j})}; i = 1,2, \ldots, d; j = 0, 1, \ldots, \lfloor \sqrt{T/d}^{1-\e} \rfloor \dfn k
\ee
for some fixed $\e > 0$.  Then, the elements of $\psi_{i,j}$ are defined such that 
\be
 p_{i, j+1} - p_{i, j} = \sqrt{\frac{d}{T}}^{1-\e} \dfn \delta; \forall i, j
\ee
with $p_{i,0} = 0$. In this step we can use all $k+1$ values of $\psi_{i,j}$.  For the subsequent regions of the bound,
we omit the first and the last $i$ ($i=0$ and $i=k$).  Note that we can similarly define the grid to be uniformly spaced 
w.r.t.\ $\psi_j$, where spacing is $\delta \cdot (\log T)$.  This makes the distinguishability a little more tedious to prove, but does
preserve a uniform grid (which will be necessary for $L_2$ and $L_1$ results).
For simplicity, and without loss of generality, we will show distinguishability with the current 
definition.

The setup above transformed the problem to a standard well-known universal compression problem \cite{rissanen84}, but
for $d$ different segments,
where in each a single parameter is to be estimated.  The cardinality of the grid is $M = k^d$.  Distinguishability is proved
using the union bound on the $d$ segments applying large deviations typical sets analysis (see \cite{cover06}).  We skip this step,
and perform it for the next regions.  The method we use for the next regions also applies here.  Applying 
Theorem~\ref{the:red-cap} with 
a fixed $x^{*T}$ gives the lower bound for the first region, taking the logarithm of $M$.

We now consider larger $B = \gamma \log T$, and the behavior with asymptotically larger $d$.  Where for simplicity, we assume that
$\gamma$ satisfies $\eref{eq:gamma_def}$.  (If it does not, we lower the value of $B$ for which the analysis is done.)
If $d > T^{1-\e}$, the analysis will assume $d = T^{1-\e}$, and the resulting bound will be also applied for large $d$.
In single dimension, for a sequence of
length $T$, we can only distinguish between parameters in $[-0.5 \log T, 0.5 \log T]$.  Any parameter greater than $0.5 \log T$ will appear
with non-diminishing probability the same as the $0.5 \log T$ parameter.  Similarly, parameters smaller than $-0.5 \log T$ will not
be distinguishable from $-0.5 \log T$.  Therefore, we cannot use the method described for smaller $B$ to enhance the bound for the larger
$B$.  We will thus have to manipulate the wider region to achieve the desired results.  The idea is to ``sacrifice'' one dimension from
the parameters to serve as a prior, which for the other parameters maps different regions in
$[-B, B]$ to $[-0.5 \log T, 0.5 \log T]$.  Each such distinct region will
be considered in a separate segment, in which distinguishability will be shown for parameter values in that region.  We note
that because we actually consider intervals $[-0.5 \log (T/d), 0.5 \log (T/d)]$, instead of a factor $\gamma$, we would actually have a
factor $\gamma \cdot \kappa$ in the lower bounds, where $\kappa \dfn (\log T) / \log(T/d) \geq 1$.  For simplicity, however, we chose to omit
this term from the bounds.  We proceed without this sidestep.

Recall that $\psi_{i,j}$ is the $j$th grid point for the value of parameter $i; 1\leq i \leq d$.  We now omit the extreme points $j=0$ and $j=k$
to prevent duplications between partitions of the region $[-B, B]$.  For $i=1$, we now only include one grid point $\psi_{1,1} = B$.
Instead of $d$ segments as before, we now have $d-1$ segments, one for each of the remaining $d-1$ components of $\psi$.
We further segment each of the $d$ segments of $x^{*T}$ defined earlier, each into $2\gamma + 1$ subsegments.  (For simplicity,
we ignore negligible integer length constraints on $\gamma$.)  Let $s \in \{-\gamma, -\gamma + 1, \ldots, 0, 1, \ldots, \gamma \}$
be a subsegment index.  Then, for subsegment $s$, $x_{t,1}^* = s / \gamma$.  We still have for segment $\ell$ representing
dimension $\ell$, $x_{t, \ell}^* = 1$ for one $\ell$,
and for the remaining components $j \neq 1, j \neq \ell$, $x_{t,j}^* = 0$.   The grids $\psi_{i, j}$ for dimensions $i = 2, 3, \dots, d$
now consist
of a grid which is a union of $2\gamma +1$ sub-grids.  The $s$th sub-grid of dimension $i$ is on the range
$[-0.5 \log T, 0.5\log T] - s \log T$.  By adding the contribution of component $1$, which is $B \cdot s / \gamma = s \log T$, the
region of sub-grid $s$ in subsegment $s$ is mapped back into $[-0.5 \log T, 0.5\log T]$.  With this mapping in mind, we place the points
in the sub-grid $s$ such that they either map into probabilities as in \eref{eq:grid_prob}, with the contribution of component $1$, or
they are just uniformly spaced in $[-0.5 \log T, 0.5\log T] - s \log T$.  The spacing in each sub-grid must be adjusted (increased) 
from the case of first region of the bound to account for
the reduction in subsegment length, and is set to 
\be
 \delta = \sqrt{\frac{d \gamma}{T}}^{1-\e}.
\ee

The construction described yields a total $(d-1)(2\gamma+1)$ subsegments, each of length $T / [(d-1)(2\gamma+1)]$, with a total
of $(2 \gamma + 1) / \delta$ grid points in each of the $d-1$ original segments.  Thus
\bea
 \nonumber
 \log M &=& (d-1) \left [ \log (2 \gamma + 1) + (1-\e) \frac{1}{2} \log \frac{T}{d\gamma} \right ] \\
 \label{eq:lb_grid_logsize}
 &\geq& (1 - o(1)) \frac{d}{2} \log \frac{4 \gamma T}{d}.
\eea

Let $S_T$ be generated by $\t \in \grid$, and let $\hat{\t}$ be estimated from $S_T$.  Define $\delta_e = \hat{\t} - \t$.
For making an error between two grid points we have to have $\delta_e \geq 0.5 \sqrt{d \gamma / T}^{1 - \e}$.  
Using large deviation
typical sets analysis, there are at most $T / [(d-1)(2\gamma+1)]$ different types for an error event per subsegment.  Then using the union
bound on $2\gamma+1$ subsegments, and then again on $d-1$ segments, we get a multiplier of $T$.  Thus, absorbing lower order
terms in $o(1)$, we have
\bea
 \nonumber
 P_e &\leq& (1 + o(1)) \cdot T \cdot \exp \left \{  -\frac{T}{2 d \gamma} \min_{\hat{\t} \neq \t} D(P_{\hat{\t}} || P_\t ) \right \} \\
 \label{eq:error_prob_lb}
 &\leq&
 (1 + o(1)) \cdot \exp \left \{ \log T - \frac{1}{4} \left ( \frac{T}{\gamma d} \right )^\e \right \}
\eea
where the second inequality follows from the relation between the KL divergence and $L_1$ norm
$D(P_{\hat{\t}} || P_\t ) \geq 2 \delta_e^2$ taking the minimum $\delta_e$ for an error event.  From the definition of $\gamma$ in
\eref{eq:gamma_def} and from the assumption made that $d \leq T^{1-\e}$, we
have $\gamma d = o(T)$, which with a fixed $\e > 0$ yields $P_e = o(1)$.  We note that this can also be achieved with diminishing
$\e = O( \log \log T / (\log T))$, with large enough constant.  This, together with the bound of \eref{eq:lb_grid_logsize} and
Theorem~\ref{the:red-cap}, concludes the proof for the second region.  The derivation of the error for the first region is very 
similar to the one above.
The third region is proved by taking $d$ that maximizes the bound in the second region and applying its bound to every
greater value of $d$.  (This is the justification for replacing $d > T^{1-\e}$ earlier by $T^{1-\e}$.)

It remains to derive the lower bounds for $L_2$ and $L_1$.  Utilizing a uniform version of the grid for $L_\infty$, and since we 
already proved distinguishability, this remains as a counting problem, of which portions of $\Psi$ satisfy the $L_2$ and $L_1$ 
constraints.  For $L_2$, using the volume of a $d$-dimensional ball, we have
\bea
 \nonumber
 \log M &=& (1 - o(1)) \cdot \left ( \frac{d}{2} \log \pi + d \log \gamma - \frac{d}{2} \log \frac{d}{2e} + \frac{1-\e}{2} d \log \frac{T}{d\gamma}  \right )\\
 \label{eq:L2_lb_M}
 &\geq&
 (1 - o(1)) \frac{d}{2} \log \frac{2 \pi e \gamma T^{1-\e}}{d^2}.
\eea
This gives the fourth region of the bound. 

We observe that as $d = \T(\sqrt{T})$ the bound becomes negative.  This is because we can no longer fit $d$-dimensional cubes 
in the $d$-dimensional balls.  Instead, cubes with lower dimensions can be fit in a lower dimension ball.  This implies that we can
have grid points that are sparse in the sense of having many $0$ components.  The set of possible points for large dimensions
can include all the combinations of points in lower dimensions.  We can thus lower bound the regret by taking the dimension that 
maximizes the lower bound in \eref{eq:L2_lb_M}.  The value $d = \sqrt{2 \pi \gamma T^{1-\e} / e}$ maximizes the bound.  Plugging
it into \eref{eq:L2_lb_M} gives the fifth region of the bound.

Similarly, normalizing the size of the grid by $d!$ gives a lower bound on $M$ for $L_1$
\be
 \label{eq:L1_lb_M}
 \log M \geq (1 - o(1)) \frac{d}{2} \log \frac{4 e^2 \gamma T^{1-\e}}{d^{3-\e}}.
\ee
This gives the sixth region of the bound.  Again, a similar issue as for $L_2$ occurs, but now at $d = O(T^{1/3})$.  Optimizing, again,
for a lower value of $d$, gives an optimizer at $d = \left(4 \gamma T^{1-\e} /e \right )^{1/3}$.  Plugging this value to \eref{eq:L1_lb_M}
gives the last region of the bound, thus concluding the proof.
\end{proof}

\section{Proof of Theorems~\ref{the:bayes_ub} and~\ref{the:bayes_ub1}}
\label{ap:ub_proof}

\begin{proof}{of Theorems~\ref{the:bayes_ub} and~\ref{the:bayes_ub1}:}
We apply Theorem~\ref{the:variational}.
For $L_\infty$ and $L_2$, we extend \cite{kakade05}, using Gaussian distributions with a diagonal covariance matrices for both the 
prior $p_0$ and $Q$.
For $L_1$, the Gaussian distributions cannot work, and we use a uniform prior $p_0$ on a grid with $Q$ with diagonal covariance.
(This also works for the other two cases,
and gives identical bound for $L_\infty$ but a weaker bound for $L_2$.)  The first term in \eref{eq:var_ub} dominates for $L_\infty$, 
and the first regions of $L_1$ and $L_2$, but the second term for the second regions of $L_1$ and $L_2$.

Let $p_0 \dfn \mathcal{N} (\t; 0, \nu^2 I_d)$, $0$-mean normal with diagonal covariance with $\nu^2$ variance.
Let $Q(\t) \dfn \mathcal{N} (\t; \t^*, \e^2 I_d)$ be a normal distribution with $\t^*$ mean and diagonal covariance
with variances $\e^2$. Then, as \cite{kakade05} showed
\be
 D(Q||p_0)  = d\log \nu + \frac{1}{2\nu^2} \left (\|\t^* \|_2^2 + d\e^2 \right ) - \frac{d}{2} - d \log \e. 
\ee
By definition of $Q$, the second term in \eref{eq:var_ub} is $T d \e^2 / 8$ (where $\eta_q^2 = \e^2$).  Combining the terms,
minimizing for $\e$, we have $\e^2 = 4\nu^2 / (4 + T\nu^2)$.  (This is slightly different from \cite{kakade05}, because the $d$
term is omitted due to the different constraints on the norm of $x_t$.)  Plugging $\e$,
\be
 \R (\A^*, S_T, \t^* ) \leq \frac{1}{2\nu^2} \|\t^2\|_2^2 + \frac{d}{2} \log \left ( 1 + \frac{T \nu^2}{4} \right ).
\ee

Extending \cite{kakade05}, we now consider what variance of $p_0$ would give the smallest bound.  This is achieved
with $\nu^2 = \|\t^*\|_2^2 / d$, which implies that we must have large variance on $p_0$ so that the cost of the prior does not
dominate the bound.  This may be unexpected, but in order to encapsulate the whole allowed range $\T$, it is reasonable that the
prior has variance, large enough, to include the far ends.  For the worst case with $L_\infty$, $\|\t^*\|_2^2 = d B^2$, giving $\nu^2 = B^2$.
For $L_2$, $\|\t^*\|_2^2 = B^2$, giving $\nu^2 = B^2 /d$.  In both cases, the first term of the bound becomes $d/2$.  Plugging $\nu^2$
to the second term as well, combining both terms into the logarithm, gives the first two regions of the bound.

For $L_1$, we use a uniformly distributed grid for the support of $p_0$ which is defined by
\[
\grid = \T_m = \left \{\psi : \begin{array}{l}
  \|\psi\|_1 \leq B; \\
  \psi_{j,i} = i \cdot \e;\\
  ~~~~~~~~~ i \in \{-\lfloor B/\e\rfloor-1, -\lfloor B/\e\rfloor,  -\lfloor B/\e\rfloor + 1, \ldots
 \lfloor B/ \e \rfloor, \lfloor B/ \e \rfloor  +1\}; \\
 ~~~~~~~~~\forall j \in \{1,2,\ldots, d\} \end{array} \right \}
\]
where $j$ denotes the dimension. The grid consists of $\e$ spaced points in each dimension, including $0$, in $[-B-\e, B+\e]$, that satisfy
the $L_1$ constraints.  The spacing parameter $\e$ will be optimized later.
Allowing for integer length constraints and accounting for $0$,
in each dimension, we upper bound the number of grid points by $2 (B + \e) / \e + 1$, giving an additional $\e$ margin (that would not
actually matter for the asymptotic results).
(Note that the points in $Q$ must be in $\T_m$, so that $D(Q||p_0)$ is finite.  Hence, the margin outside $[-B,B]$ is needed.)
The volume of a subspace of a cube in $\Rspace^d$
which satisfies an $L_1$ constraint is the $d!$ fraction of the space.
Hence, to bound the actual number of grid points we divide the number of points in the cube by $d!$.  For sufficiently small $d= o(B\sqrt{T})$, this will 
suffice with the extra margin.  Hence, we have
\be
 \label{eq:ub_grid_M}
 M \leq \frac{\left ( \frac{2(B+\e)}{\e} + 1 \right )^d}{d!}.
\ee

Assume that the $i$th dimension $\t^*_i$ of $\t^*$ falls between adjacent grid points $\zeta_1$ and $\zeta_2$, for which
$\zeta_2 - \zeta_1 = \e$ in the
projection of $\grid$ to dimension $i$.
We then define the distribution Q as a product of independent components
\be
 \label{eq:Q_grid}
 Q(\t) \dfn \prod_{i=1}^d q_i(\t_i)
\ee
where
\be
 q_i (\t_i) \dfn \left \{
 \begin{array}{ll}
 \alpha; & \t_i = \zeta_1, \\
 1 - \alpha; & \t_i = \zeta_2, \\
 0; & \text{otherwise} 
 \end{array}
 \right .
\ee
where $0 \leq \alpha \leq 1$ is determined such that
\be
 E_{q_i} (\t_i) = \alpha \zeta_1 + (1 - \alpha) \zeta_2 = \t_i^*.
\ee
By definition of $\zeta_1$, $\zeta_2$ and the expectation of $q_i(\t_i)$, we have $\alpha = (\zeta_2 - \t_i^*)/\e$, and we can show
that the variance of $q_i(\cdot)$ is
\be
 \label{eq:Q_var_grid}
 E_{q_i}(\t_i - \t_i^*)^2 = (1-\alpha)\alpha \e^2 \leq \frac{\e^2}{4} \dfn \eta_q^2.
\ee

We can now apply the bound of Theorem~\ref{the:variational}.  Since $p_0(\cdot)$ is uniform over the grid with $M$ points, and the support of $Q(\cdot)$
is a subset of the support of $p_0(\cdot)$,
\be
 D(Q||p_0) \leq \log M
\ee
where the negative entropy of $Q(\cdot)$ is bounded by $0$ (which can be the case if $\t^*$ falls on a grid point $\psi \in \grid$).
The second term of \eref{eq:var_ub} is bounded from \eref{eq:Q_var_grid}.
Substituting \eref{eq:ub_grid_M} for $M$ (absorbing low order terms in the $o(1)$ term) and using Stirling approximation for the factorial,
we thus have
\be
 \R \left (\A^*, S_T, \t^* \right ) \leq  (1 + o(1)) \left [
 d \log (2B) - d \log\frac{d}{e} - \frac{1}{2} \log (2\pi d) - d \log \e + \frac{T d \e^2}{32}
 \right ].
 \ee
 Differentiating w.r.t.\ $\e$ gives $\e^2 = 16/T$, which gives the minimal bounds.  Substituting this value, gives the third region of the
 bound in Theorem~\ref{the:bayes_ub}.
 
 As long as $d = o(B\sqrt{T})$, the bound on $M$ includes sparse points with many components which are $0$ (this is guaranteed by
 the additional margins and $1$ term in \eref{eq:ub_grid_M}, which are basically negligible in this region).
However, for larger $d$, we no longer have full $d$-dimensional cubes with side $\e$ that can fit in the allowable volume of the $L_1$
constrained space.  We still, however, have points that are included in this space, that have a large fraction of $0$ coordinates, but for which
the bound in \eref{eq:ub_grid_M} is no longer sufficient.
We note that this problem is not only an artifact of the grid approach.  The first term of \eref{eq:var_ub} becomes negligible for $d = O(T)$
with $L_2$, and if we used $\|x_t\|_2 \leq 1$ in our setting, this would also happen at $d= O(T)$ for $L_\infty$, and at $d = O(\sqrt{T})$ for
$L_2$.  This implies that because the constraints shrink the parameter space, the variance of $p_0$ will be small enough, such that $Q$
and $p_0$ almost match.

When the bound in \eref{eq:ub_grid_M} starts diminishing, the normalization in $d!$ eliminates many grid points that
include $0$ coordinates from the count.  To account for these points, we can upper bound $M$ with a union bound over 
all subsets of $d$ with $n$ nonzero coordinates.
\be
 \label{eq:ub_grid_largeM}
M \leq (1 + o(1)) \cdot \sum_{n=1}^d \comb{d}{n} \frac{B^n 2^n}{n! \e^n}.
\ee
The term to the right of the combination is maximized for $n_o = 2B/\e = B\sqrt{T}/2$.   For $d = \Theta (B\sqrt{T})$, $n_o = \Theta(d)$.
We can use a (loose) upper bound of $2^d$ on the combination number in \eref{eq:ub_grid_largeM}, to obtain a union bound
\be
 \log M \leq (1 + o(1)) \cdot \left [
 d \log 2 + \frac{B\sqrt{T}}{2} + \log d
 \right ],
\ee
Combining this bound with the bound of $d/2$ on the right term of \eref{eq:var_ub} we obtain the fourth region of the bound of 
Theorem~\ref{the:bayes_ub}.

To prove the last region, we can use a tighter bound on the combination term in \eref{eq:ub_grid_largeM}, for $n_o \leq d/2$.  For the
last region this is satisfied.  We thus bound
\be
M \leq (1 + o(1)) \cdot \sum_{n=1}^d \left ( \frac{d e}{n}\right )^n \frac{B^n 2^n}{n! \e^n}.
\ee
The largest element for the sum is obtained with $n_0^2 = d B \sqrt{T}/2 = o(d)$ in this region.  Plugging $n_o$, still using $\e = 4/\sqrt{T}$,
using a union bound on all $d$ elements of the sum, taking the logarithm of $M$, and adding the $d/2$ bound for the right term of 
\eref{eq:var_ub} gives the last region of the bound of Theorem~\ref{the:bayes_ub}, thus concluding its proof.

The only change from the last region to prove Theorem~\ref{the:bayes_ub1} is that now we first bound $M$ using the 
optimizing $n_o^2 = 2 B d / \e$, and then we find $\e$ that minimizes the joint bound.  Bounding all terms with a parameter $\e$ gives
\be
 \R \left (\A^*, S_T, \t^* \right ) \leq  (1 + o(1)) \cdot \left [ \log d + \sqrt{\frac{8 B d}{\e}} + \frac{Td\e^2}{32}
 \right ].
\ee
The minimizing $\e$ is then
\be
 \e = \frac{2^{9/5} B^{1/5}}{T^{2/5}d^{1/5}}.
\ee
Plugging this value in the bound gives
\be
 \R \left (\A^*, S_T, \t^* \right ) \leq  (1 + o(1)) \cdot \frac{5}{4} \cdot 2^{3/5} B^{2/5} d^{3/5} T^{1/5} = o(d).
\ee
This concludes the proof of Theorem~\ref{the:bayes_ub1}.

We note that the grid approach used for the $L_1$ bounds requires an algorithm to know the horizon $T$ in advance, to perform the
mixture.  However, in a \emph{strongly sequential\/} setting, when the horizon is not known \emph{a-priori\/}, we can start with some 
hypothesized horizon.  Once it is reached, the next horizon can be squared (or exponentiated with some exponent
$1 + \e$, for some small $\e$), and the current posterior (or prior of the next example), can
be split from each grid point to the new $T^\e$ nearest grid points.  This will incur some additional negligible regret relative to the
logarithmic regret in $T$, but can still achieve the current bounds.
\end{proof}

\section{Multi Label Lower Bound}
\label{ap:multi}

Similarly to Theorem~\ref{th:binary_lb}, we can state a regret lower bound for multi label logistic regression.  We only state the
bound for $L_\infty$.  Let $\t^{*(m)}$ be the projection of the parameter space on label $m$, i.e., a $d$-dimensional parameter space
for label $m$.
\begin{theorem}
\label{th:multi_lb}
Let $\|\t^{*(m)}\|_\infty \leq B$.
Fix an arbitrary $\e > 0$,  let $T\rightarrow \infty$, and let the number of labels $m = o((T/d)^{1-\e})$.
Then, for every algorithm $\A$ there exists a sequence $S_T$, for
which the regret for multi $m$ label logistic regression is lower bounded by
\be
 \label{eq:multi_mix_lb}
  \R (\A, S_T) \geq (1 - o(1)) \cdot \frac{d (m-1)}{2} \log \frac{T}{m \cdot d}.
\ee
\end{theorem}

The $m-1$ factor is used to indicate that there are only $m-1$ free parameters per feature.  It does not affect the asymptotic behavior.  The proof
of Theorem~\ref{th:multi_lb} is similar to that of the first region of Theorem~\ref{th:binary_lb} in segmenting $x^{*T}$ into
$d$ segments, where in each only a single dimension exists.  However, handling each segment, especially if $m = \Omega(1)$,
is substantially more difficult.  The behavior for a single segment, however, appears as Theorem~1 in \cite{shamir06}.  For tight behavior it requires a nonuniform 
grid, which is described in the proof.  Because of the nonuniform grid, distinguishability (proven in Appendix~A in \cite{shamir06}) is more difficult to prove.  Borrowing
the proof from \cite{shamir06}, all is left to do is sum up the size of the grid over the $d$ segments, and apply the union bound on the $d$ segments for
proving distinguishability.  In each segment, there are $\log M_d = (1 - o(1)) 0.5 m \log T / (md)$ distinguishable parameters.  The normalization in
$d$ is because the segment is of length $T/d$ by the partitioning of $x^{*T}$.  The $m$ parameter appears due to effectively $L_1$ constraints imposed by the fact that the probabilities on all labels must sum to $1$.  As in the proof of Theorem~\ref{th:binary_lb}, for
$d  > T^{1-\e}$, we clip the analysis at $d_m = T^{1-\e}$.  As in Theorem~\ref{th:binary_lb}, we will have a threshold that depends on both
$d$ and $m$ over which the bound becomes negative and useless.  We can derive bounds, as those in the other regions of Theorem~\ref{th:binary_lb} for these cases, by taking the values of $d$ and $m$ that produce a maximal value of the bound, and lower bounding the regret for the larger $d$ and $m$ by the regret for the maximizing $d$ and $m$.

\end{document}